\documentclass{ecai}
\usepackage{times}
\usepackage{graphicx}
\usepackage{latexsym}

\usepackage{amsmath}
\usepackage{amssymb}
\usepackage{amsthm}
\usepackage{lipsum}

\usepackage{graphicx}
\usepackage{amsthm,amsmath}
\usepackage{enumitem}
\usepackage{pifont}
\usepackage{tikz}
\usetikzlibrary{arrows}

\usepackage{amsthm}
\usepackage{amsmath}
\usepackage{amssymb}

\usepackage{cite}

 \newcommand{\OMIT}[1]{} %

\newcommand\qedblob{\ding{113}}
\def\literalqed{{\ \nolinebreak\hfill\mbox{\qedblob\quad}}}

\theoremstyle{plain}

\newtheorem{proposition}[theorem]{Proposition}

\theoremstyle{definition}
\newtheorem{definition}[theorem]{Definition}

\newtheorem{example}[theorem]{Example}

\theoremstyle{remark}

\definecolor{verydarkgrey}{gray}{0} 
\definecolor{grey}{gray}{0} 
  
\DeclareMathAlphabet{\mathpzc}{OT1}{pzc}{m}{it}

\newcommand{\cala}{\ensuremath{\mathcal{A}}}
\newcommand{\calr}{\ensuremath{\mathcal{R}}}
\newcommand{\safeq}{\ensuremath{AF =\langle\cala,\calr\rangle}}
\newcommand{\Af}{\ensuremath{\langle\cala,\calr\rangle}}
\newcommand{\AF}{\ensuremath{AF}}

\newcommand{\entailsZ}{\,{|\hspace{-0.49em}\sim}_{\text{Z}}\,}	

\usepackage{todonotes}

\begin{document}

\title{Towards Ranking-based Semantics for Abstract Argumentation using Conditional Logic Semantics}

\author{Kenneth Skiba\institute{Institute for Web Science and Technologies, University of Koblenz-Landau, Germany, email: \{kennethskiba, thimm\}@uni-koblenz.de} \and Matthias Thimm$^1$}

\maketitle
\bibliographystyle{ecai}
\begin{abstract}
We propose a novel ranking-based semantics for Dung-style argumentation frameworks with the help of conditional logics. Using an intuitive translation for an argumentation framework to generate conditionals, we can apply nonmonotonic inference systems to generate a ranking on possible worlds. With this ranking we construct a ranking for our arguments. With a small extension to this ranking-based semantics we already satisfy some desirable properties for a ranking over arguments. 
\end{abstract}

\section{Introduction}
Formal argumentation \cite{Atkinson:2017} describes a family of approaches to modeling rational decision-making through the representation of arguments and their relationships. A particular important representative approach is that of abstract argumentation \cite{dun:j:argument-acceptability}, which focuses on the representation of arguments and a conflict relation between arguments through modeling this setting as a directed graph. Here, arguments are identified by vertices and an \emph{attack} from one argument to another is represented as a directed edge. This simple model already provides an interesting object of study, see  \cite{Baroni:2018} for an overview. Reasoning is usually performed in abstract argumentation by considering \emph{extensions}, i.\,e., sets of arguments that are jointly acceptable given some formal account of ``acceptability''. Therefore, this classical approach differentiates between ``acceptable'' arguments and ``rejected'' arguments. However, empirical cognitive studies such as the ones described in \cite{Rahwan:2010a,Polberg:2018a} have shown that humans assess arguments in a more fine-grained manner. For example, while \cite{Rahwan:2010a} provides evidence that while humans adopt the ``reinstatement principle'' (which states that arguments defended by accepted arguments should also be accepted), they usually assign lower confidence to reinstated arguments than non-attacked ones. Similarly, the experiments described in \cite{Polberg:2018a} advocate that a probabilistic interpretation of arguments \cite{Hunter:2020} is more suitable than the classical two-valued interpretation.

In order to formally address the observations described above, ranking-based semantics \cite{amg-ben:c:ranking-based-semantics} provide a fine-grained assessment of arguments. Here, we follow this line of work and make some first steps towards the use of conditional logic and the System Z inference mechanism \cite{goldszmidt1996qualitative} to define rankings between arguments. Conditional logic is a general non-monotonic representation formalism that focuses on default rules of the form ``if A then B'' and there exist some interesting relationships between this formalism and that of formal argumentation \cite{ker-thim:2018:cl-towards-adf,Heyninck2020BetweenADF}. We make use of these relationships here for the purpose of defining a novel ranking-based semantics for abstract argumentation.

The rest of this work is organized as follows: In Section 2 all necessary preliminaries will be stated. Then we discuss our ranking idea in Section 3 and with Section 4 we conclude this paper. 
\section{Background}
%
%
In the following, we want to briefly recall some general preliminaries on argumentation frameworks and conditional logics.

\subsection{Abstract Argumentation Frameworks}

In this work, we use \emph{argumentation frameworks} first introduced in \cite{dun:j:argument-acceptability}. 
An \emph{argumentation framework} $\AF$ is a pair $\Af$, where $\cala$ is a finite set of arguments and $\calr$ is a set of attacks between arguments with $\calr \subseteq \cala \times \cala$.
An argument $a$ is said to \emph{attack $b$} if $(a,b) \in \calr$.
We call an argument $a$ \emph{acceptable with respect to a set $S \subseteq \cala$} if for each attacker $b \in \cala$ of this argument $a$ with $(b,a) \in \calr$, there is an argument $c \in S$ which attacks $b$, i.\,e., $(c,b) \in \calr$; we then say that $a$ is \emph{defended by~$c$}.
An argumentation framework $\Af$ can be illustrated by a directed graph with vertex set $\cala$ and edge set~$\calr$.

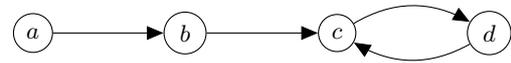
\begin{figure}[t] 
\center
\begin{tikzpicture}
\node (a1) at (0,0) [circle, draw] {$a$};
\node (a2) at (2,0) [circle, draw] {$b$};
\node (a3) at (4,0) [circle, draw] {$c$};
\node (a4) at (6,0) [circle, draw] {$d$};

\path[-triangle 45] (a1) edge  (a2);
\path[-triangle 45] (a2) edge  (a3);
\path[-triangle 45, bend left] (a3) edge  (a4);
\path[-triangle 45, bend left] (a4) edge (a3);
\end{tikzpicture}
\caption{Argumentation framework from Example~\ref{exampleAF}.}
\label{Tikz:exampleAF}
\end{figure}

\begin{example}\label{exampleAF}
Let $\safeq$ with $\cala = \{a,b,c,d\}$
and $\calr =\{(a,b),(b,c),(c,d),(d,c)\}$ be an argumentation framework.
The corresponding graph is shown in Figure~\ref{Tikz:exampleAF}.
Argument $b$ is not acceptable with respect to any set $S$ of arguments, as $b$ is not defended against $a$'s attack.  On the other hand, $c$ is acceptable with respect to $S= \{a,c\}$, as $a$ defends $c$ against $b$'s attack and $c$ defends itself against $d$'s attack.
\end{example}

Up to this point the arguments can only have the two statuses of accepted or not accepted\footnote{However, using labeling-based semantics we can generate a three-valued model \cite{wu-cam-pod:j:labelling-justification}.}, but we want to have a more fine-graded comparison between arguments. For this we use the idea of ranking-based semantics \cite{amg-ben:c:ranking-based-semantics,del:dis:ranking}.
\begin{definition}[Ranking-based semantics]
A \emph{ranking-based semantics} $\sigma$ associates to any argumentation framework \safeq\ a preorder $\succeq_{AF}^{\sigma}$ on \cala. $a \succeq_{AF}^{\sigma} b$ means that $a$ is at least as acceptable as $b$. With $a \simeq^{\sigma}_{AF} b$ we describe that $a$ and $b$ are equally acceptable, i.\,e., $a \succeq_{AF}^{\sigma} b$ and $b \succeq_{AF}^{\sigma} a$. Finally we say $a$ is strictly more acceptable than $b$, denoted by $a \succ_{AF}^{\sigma} b$, if $a \succeq_{AF}^{\sigma} b$ and not $b \succeq_{AF}^{\sigma} a$.
We denote by $\sigma(AF)$ the ranking on \cala\ returned by $\sigma$.
\end{definition}

\subsection{Conditional Logics}
We use a set of atoms $A$ and connectives $\land$ (and), $\lor$ (or), and $\neg$ (negation) to generate the \emph{propositional language} $\mathcal{L}(A)$. $w$ is an \emph{interpretation} (or \emph{possible world}) for $\mathcal{L}(A)$ when $w: A \rightarrow \{\textsc{true},\textsc{false}\}$. We denote the set of all interpretations as $\Omega(A)$. An interpretation $w$ \emph{satisfies} an atom $a \in A$ ($w \vdash a$), if and only if $w(a) = \textsc{true}$. The relation $\vdash$ is extended to arbitrary formulas in the usual way. We will abbreviate an interpretation $w$ with its \emph{complete conjunction}, i.\,e., if $a_{1}, \dots, a_n \in A$ are the atoms that are assigned $\textsc{true}$ by $w$ and $a_{n+1}, \dots, a_m \in A$ are the ones assigned with $\textsc{false}$, $w$ will be identified with $a_1 \dots a_n \overline{a_{n+1}} \dots \overline{a_{m}}$. For $\Phi \subseteq \mathcal{L}\{A\}$ we define $w \vdash \Phi$ if and only if $w \vdash \phi$ for every $\phi \in \Phi$. With $Mod(X) = \{w \in \Omega(A) | w \vdash X\}$ we define the set of models for a set of formulas $X$. A \emph{conditional} is a structure of the form $(\varphi | \phi)$ and represents a rule ``If $\phi$ than (usually) $\phi$''. 

We can consider conditionals as \emph{generalized indicator functions} \cite{finetti1974theory} for possible worlds $w$ as follows: 
\begin{align}
((\varphi | \phi ))(w) = \begin{cases} 
1 : w \vdash \phi \land \varphi \\
0 : w \vdash \phi \land \neg \varphi \\
u : w \vdash \neg \phi
\end{cases}
\end{align}
where $u$ stands for \emph{unknown}.
Informally speaking, a world $w$ \emph{verifies} a conditional $(\varphi| \phi)$ iff it satisfies both antecedent and conclusion $((\varphi | \phi )(w)=1)$; it \emph{falsifies} iff is satisfies the antecedence but not the conclusion $((\varphi | \phi )(w) = 0)$; otherwise the conditional is \emph{not applicable} $((\varphi | \phi )(w) = u)$. A conditional $(\varphi|\phi)$ is satisfied by $w$ if it does not falsify it.

Semantics are given to sets of conditionals via ranking functions \cite{goldszmidt1996qualitative,Spohn:1988}.
With a ranking function, also called \emph{ordinal conditional function (OCF)}, $\kappa: \Omega(A) \rightarrow \mathbb{N} \cup \{\infty\}$ we can express the degree of plausibility of possible worlds $\kappa (\phi) := min\{\kappa(w) | w \vdash \phi\}$. With the help of OCFs $\kappa$ we can express the acceptance of conditionals and nonmonotonic inferences, so $(\varphi| \phi)$ is accepted by $\kappa$ iff $\kappa(\phi \land \varphi) < \kappa(\phi \land \neg \varphi)$. With $Bel(\kappa)= \{\phi | \forall w \in \kappa^{-1}(0): w \vdash \phi\}$ we denote the most plausible worlds.

As there are an infinite number of ranking functions that accept a given set of conditionals, we consider System Z \cite{goldszmidt1996qualitative} as an inference relation, which yields us a uniquely defined ranking function for reasoning.
\begin{definition}[System Z]
$(\varphi|\phi)$ is tolerated by a finite set of conditionals $\Delta$ if there is a possible world $w$ with $(\phi|\varphi)(w) = 1$ and $(\phi'|\varphi')(w) \neq 0$ for all $(\phi'|\varphi') \in \Delta$. 
The \emph{Z-partition} $(\Delta_0, \dots, \Delta_n)$ of $\Delta$ is defined as:
\begin{itemize}
\item $\Delta_0 = \{\delta \in \Delta | \Delta~\text{tolerates}~\delta\}$
\item $\Delta_1, \dots, \Delta_n$ is the Z-partition of $\Delta \setminus \Delta_0$
\end{itemize}
For $\delta \in \Delta$: $Z_{\Delta}(\delta)= i$ iff $\delta \in \Delta_{i}$ and $\Delta_1, \dots, \Delta_n$ is the Z-partitioning of $\Delta$. 

We define a ranking function $\kappa^Z_{\Delta}: \Omega \rightarrow \mathbb{N} \cup \{\infty\}$ as $\kappa^Z_{\Delta}(w)= max\{Z(\delta)| \delta(w) = 0, \delta \in \Delta \}+ 1$, with $max ~\emptyset = -1$.
Finally $\Delta \entailsZ \phi$ if and only if $\phi \in Bel(\kappa^Z_{\Delta})$.
\end{definition}

\begin{example}\label{ex:cl}
Let $\Delta = \{(a|\neg b), (b|\neg a), (c|\neg b \land \neg a \land \neg d), (d|\top), (c| \neg d)\}$. For this set of conditionals, $\Delta = \Delta_0 \cup \Delta_1$ with $\Delta_0 = \{(a|\neg b), (b|\neg a), (c|\neg b \land \neg a \land \neg d)\}$ and $\Delta_1 = \{(\neg a \land \neg b|~ d)\}$ therefore we have the values from Table \ref{tab:ex:cl}.
\noindent
\begin{table*}
\caption{Values for Example \ref{ex:cl}}\label{tab:ex:cl}
\centering
\begin{tabular}{l||c|c|c|c||c}
$\omega$ & $Z((a|\neg b))$& $Z((b|\neg a))$ & $Z((c|\neg b \land \neg a \land \neg d))$& $ Z((d|\top))$ &$Z((\neg a \land \neg b| d))$ \\ \hline
$abcd$ & u & u & u & 1 & 0\\
$abc\overline{d}$ & u & u & u & 0 & u \\
$ab\overline{c}d$ & u & u & u & 1 & 0 \\
$ab\overline{c}\overline{d}$ & u & u & u & 0 & u  \\
$a\overline{b}cd$ & 1 & u & u & 1 & 0\\
$a\overline{b}c\overline{d}$ & 1 & u & u & 0 & u\\
$a\overline{b}\overline{c}d$ & 1 & u & u & 1 & 0\\
$a\overline{b}\overline{c}\overline{d}$ & 1 & u & u &0 & u  \\
$\overline{a}bcd$ & u & 1 & u & 1 & 0 \\
$\overline{a}bc\overline{d}$ & u & 1 & u & 0 & u    \\
$\overline{a}b\overline{c}d$ & u & 1 & u & 1 & 0  \\
$\overline{a}b\overline{c}\overline{d}$ & u & 1 & u & 0 & u \\
$\overline{a}\overline{b}cd$ & 0 & 0 & u & 1 & 1\\
$\overline{a}\overline{b}c\overline{d}$ & 0 & 0 & 1 & 0 & u \\
$\overline{a}\overline{b}\overline{c}d$ & 0 & 0 & u & 1 & 1\\
$\overline{a}\overline{b}\overline{c}\overline{d}$ & 0 & 0 & 0 & 0 & u
\end{tabular}
\end{table*}
So we can derive $(\kappa^Z_{\Delta_0})^{-1}(0) = \{abcd, ab\bar{c}d, a\bar{b}cd, a\bar{b}\bar{c}d,\bar{a}bcd, \bar{a}b\bar{c}d\}$ and $(\kappa^Z_{\Delta_1})^{-1}(0) = \emptyset$.
\end{example}

\section{Ranking-based Semantics with Conditional Logic Semantics}
%
In this work we want to extend previous works \cite{Heyninck2020BetweenADF,ker-thim:2018:cl-towards-adf} to not only combine abstract argumentation and conditional logics, but also present ideas to rank arguments using this combination.

The general idea is to represent an abstract argumentation framework as a set of conditionals, using System Z in order to determine a ranking function that accepts these conditionals, and then extract rankings on arguments from this ranking function. First, we need a translation from an argumentation framework to a set of conditionals. It is clear, that for an argument to be acceptable every attacker has to be not acceptable. With this idea we can construct the conditional logic knowledge base. Let $AF$ be an argumentation framework and $\theta: \mathcal{A} \rightarrow \mathcal{C}_{\mathcal{A}}$, where $\mathcal{C}_{\mathcal{A}}$ is the set of conditional knowledge bases over the propositional language generated by $\mathcal{A}$.
\begin{align} \label{func:theta}
\theta(AF) & = \{(a|B) \mid a \in \mathcal{A}, B = \bigwedge_{(b,a)\in \mathcal{R}}  \neg b \}
\end{align}
In other words, $\theta$ models that an argument is accepted if all its attackers are not accepted.

We can use inference systems like System Z on these conditional knowledge bases to generate a ranking over the possible worlds. Based on this ranking we want to rank the arguments. Our first idea is to count the number of occurrences of a positive literal in the set of worlds $(\kappa^Z_{\Delta})^{-1}(0)$ and then rank the corresponding arguments based on this number. So if an argument $a$ has a higher count then an argument $b$, we say $a \succeq b$. This simple idea yields a clear and uniquely defined ranking, while not needing a complex algorithm to be computed.  
\begin{definition}
Let \safeq\ be an argumentation framework translated with help of $\theta(AF)$ and an inference system to the set of worlds $\kappa^Z_{\Delta}$. Define
\begin{align}
Ccs_{AF}^{\theta}(a)= |\{w \in (\kappa^Z_{\theta(AF)})^{-1}(0) | w \vdash a\}|
\end{align}
\end{definition}

We can then use this counting function for our ranking-based semantics.
\begin{definition}[Conditional-counting-based semantics]

The \emph{Conditional-counting-based semantics (Ccbs)} associates to any argumentation framework \safeq\ a ranking $\succeq_{AF}^{Ccbs}$ on $\mathcal{A}$ such that $\forall a,b \in \mathcal{A}$ with respect to a translation $\theta$ and a ranking function $\kappa^Z_{\Delta}(\omega)$.
\[
a \succeq_{AF}^{Ccbs} b~ \text{if and only if} ~Ccs_{AF}^{\theta}(a) \geq Ccs_{AF}^{\theta}(b)
\]
\end{definition}

\begin{example}\label{ex:Ccbs}
Let $\safeq$ with $\cala = \{a,b,c,d\}$
and $\calr =\{(a,b),(b,a),(a,c),(b,c), (d,c)\}$ be an argumentation framework.
The corresponding graph can be found in Figure \ref{Tikz:exampleCcbs}. Using Equation \ref{func:theta} we obtain $\Delta = \{(a|\neg b), (b|\neg a), (c|\neg b \land \neg a \land \neg d), (d|\top)\}$. With $\Delta = \Delta_0$ we have $(\kappa^Z_{\Delta})^{-1}(0) = \{abcd, ab\bar{c}d, a\bar{b}cd, a\bar{b}\bar{c}d,\bar{a}bcd, \bar{a}b\bar{c}d\}$. Now we can count the number of occurrences of each argument. So $Ccs_{\kappa^Z_{\Delta}(\omega)}^{\theta}(a) = 4$, $Ccs_{\kappa^Z_{\Delta}(\omega)}^{\theta}(b) = 4$, $Ccs_{\kappa^Z_{\Delta}(\omega)}^{\theta}(c) = 3$ and $Ccs_{\kappa^Z_{\Delta}(\omega)}^{\theta}(d) = 6$. This results in $d \succeq^{Ccbs} a \simeq^{Ccbs} b \succeq^{Ccbs} c$. 

Looking at the graph we see, that argument $d$ is unattacked, so it is intuitive that this argument is ranked at the highest position. Also the arguments $a$ and $b$ are attacking each other and are not attacked by any other argument. These two arguments are there indistinguishable and should be ranked on the same level, but both arguments have at least one attacker so it should be ranking lower then $d$. Argument $c$ is attacked by three other arguments and defended by none, hence this argument should be ranked lower then its attackers.
 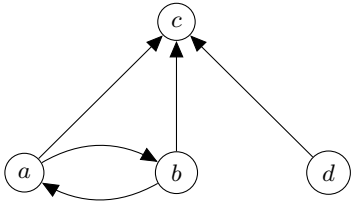
\begin{figure}[t] 
\center
\begin{tikzpicture}
\node (a) at (0,0) [circle, draw] {$a$};
\node (b) at (2,0) [circle, draw] {$b$};
\node (c) at (2,2) [circle, draw] {$c$};
\node (d) at (4,0) [circle, draw] {$d$};

\path[-triangle 45, bend left] (a) edge  (b);
\path[-triangle 45, bend left] (b) edge  (a);
\path[-triangle 45] (a) edge  (c);
\path[-triangle 45] (b) edge  (c);
\path[-triangle 45] (d) edge  (c);
\end{tikzpicture}
\caption{Argumentation framework from Example~\ref{ex:Ccbs}}
\label{Tikz:exampleCcbs}
\end{figure}
\end{example}

%
For some further ideas of other translations we recommend \cite{Heyninck2020BetweenADF}. Instead of System Z we could also use c-representations \cite{kern2001conditionals}.

Ranking-based semantics are usually evaluated wrt.\ a series of rationality postulates \cite{amg-ben:c:ranking-based-semantics,del:dis:ranking}. In this work, we provide some first steps in this direction and look at four simple ones, namely \emph{Abstraction} \cite{amg-ben:c:ranking-based-semantics}, \emph{Independence} \cite{mat-ton:c:game-theoretic-argument-strength}, \emph{Void Precedence} \cite{mat-ton:c:game-theoretic-argument-strength, amg-ben:c:ranking-based-semantics} and \emph{Self-Contradiction} \cite{mat-ton:c:game-theoretic-argument-strength}. 
With the property of \emph{Abstraction} we can ensure, that a ranking over arguments only depends on the attacks between arguments and not on the identity of the arguments.

\begin{definition}
An isomorphism $\gamma$ between two argumentation framework \safeq\ and $AF' = \langle \cala', \calr' \rangle$ is a bijective function $\gamma: \cala \rightarrow \cala'$ such that $\forall x,y \in \cala, (x,y) \in \calr$ if and only if $(\gamma(x),\gamma(y)) \in \calr'$. With a slight abuse of notation, we will note $AF' = \gamma(AF)$.
\end{definition}

\begin{definition}[Abstraction]
A ranking-based semantics $\omega$ satisfies \emph{Abstraction} iff for any \emph{AF, AF'}, for every isomorphism $\gamma$ such that $AF' = \gamma(AF)$, we have $x \succ^{\sigma}_{AF}$ iff $\gamma(x) \succ^{\sigma}_{AF}  \gamma(y)$.
\end{definition}

It is natural, that arguments from two different disconnected subgraphs should have no influence on each other for a ranking. A ranking, which satisfies \emph{Independence}, ensures this idea.

\begin{definition}
The connected components of an argumentation framework $AF$ are the set of largest subgraphs of $AF$, denoted by $cc(AF)$, where two arguments are in the same component of $AF$ iff there exists some path between them.
\end{definition}

\begin{definition}[Independence]
A ranking-based semantics $\omega$ satisfies \emph{Independence} iff for any argumentation framework $AF$ such that $\forall AF' \in cc(AF)$, $\forall x,y \in Arg(AF')$, $x \succ^{\sigma}_{AF'} y$ iff $x \succ^{\sigma}_{AF} y$.
\end{definition}

\begin{proposition}
\emph{Ccbs} satisfies \emph{Abstraction} and \emph{Independence}.
\end{proposition}

\begin{proof}
For \emph{Abstraction} we can see, that using an isomorphism does not change the structure of an argumentation framework nor does it change relationships between arguments. \emph{Ccbs} does not change if an argument $a$ is renamed to $c$ as long as incoming and outgoing attacks are still the same. 
 
Adding independent arguments to an argumentation framework does not change the ranking between two arguments. Given an argumentation framework $AF$ with two arguments $a,b$ and $a \succ^{ccbs}_{AF} b$. If we add an argument $c$ to this $AF$ to create $AF'$, at most we would change the $Ccs$-score of $a$ and $b$ by two, but this change takes place for both. So if $Ccs_{AF}^{\theta}(a) \geq Ccs_{AF}^{\theta}(b)$, then it holds that $Ccs_{AF'}^{\theta}(a) \geq Ccs_{AF'}^{\theta}(b)$. 
\end{proof}

The idea of \emph{Void Precedence} states that a non-attacked argument should be strictly more acceptable than an attacked argument. 
\begin{definition}[Void Precedence]
A ranking-based semantics $\sigma$ satisfies \emph{Void Precedence} if and only if for any \safeq\ and $\forall a,b \in \cala\ $, if $\forall c \in \cala\ $ $(c,a) \notin \calr\ $ and $\exists d \in \cala\ $ with $(d,b)\in \calr\ $, then $a \succ^{\sigma}_{AF} b$.
\end{definition}

On the contrary, a self-attacking argument should always be ranked worse than any other argument, because these arguments are contradicting themselves. This is handled with the property \emph{Self-Contradiction}.
\begin{definition}[Self-Contradiction]
A ranking-based semantics $\sigma$ satisfies \emph{Self-Contradiction} if and only if for any \safeq\ and $\forall a,b \cala$, if $(a,a) \notin \calr$ and $(b,b) \in \calr$ then $a \succ^{\sigma}_{AF} b$.
\end{definition}

\begin{proposition}
\emph{Ccbs} does not satisfy \emph{Void Precedence} nor \emph{Self-Contradiction}.
\end{proposition}

\begin{proof}
To prove this we look at the following example.
Let $\safeq$ with $\cala = \{a,b\}$ and $\calr =\{(a,a)\}$ be an argumentation framework. Using Equation \ref{func:theta} we obtain $\Delta = \{(a|\neg a), (b|\top)\}$. With $\Delta = \Delta_0$ we have $(\kappa^Z_{\Delta})^{-1}(0) = \{ab\}$ and $Ccs_{\kappa^Z_{\Delta}(\omega)}^{\theta}(a)= 1$, $Ccs_{\kappa^Z_{\Delta}(\omega)}^{\theta}(b)= 1$. This results in $a \simeq^{Ccbs}_{AF} b$. Therefore $a$ is not strictly less acceptable then $b$. So \emph{Ccbs} does not satisfy Void Precedence nor Self-Contradiction.
\end{proof}

Hence this semantics has a few shortcomings, we propose an extension. Before we count the occurrences we rank every argument with an selfattack at the lowest possible position. 
\begin{definition}
\emph{Ccbs'} associates to any argumentation framework \safeq\ a ranking $\succeq_{AF}^{Ccbs'}$ on $\mathcal{A}$ such that $\forall a,b \in \mathcal{A}$ with respect to a translation $\theta$ and a ranking function $\kappa^Z_{\Delta}(\omega)$.
\begin{align*}
if~ (a,a) \notin \calr ~ and ~(b,b) \in \calr ~ then ~a \succ^{\sigma}_{AF} b \\
otherwise~ a \succeq_{AF}^{Ccbs'} b~ \text{if} ~Ccs_{AF}^{\theta}(a) \geq Ccs_{AF}^{\theta}(b)
\end{align*}
\end{definition}
When we evaluate this ranking semantics it is easy to see that \emph{Ccbs'} satisfies \emph{Self-Contradiction} (we omit the proof). 
\begin{proposition}
\emph{Ccbs'} satisfies \emph{Void Precedence}.
\end{proposition}
%

So with this small extension we now satisfy an additional two very intuitive properties for ranking-based semantics. We leave an investigation of further properties for future work.
\section{Conclusion}
In this work we have presented a first idea to rank arguments with conditional logics. For this we first looked at a simple translation 
from an argumentation framework to conditional logic and applied an inference relation. Using a simple counting idea results in a ranking over arguments. 

Although this semantics does not satisfy two desired properties, with a small extension we have shown that these two properties are satisfied. Also we have established a simple connection between ranking arguments and conditional logic. In the future we can improve this idea and hopefully present a ranking-based semantics, which satisfies a good number of properties presented in \cite{del:dis:ranking}.

Another future work approach is to look at other frameworks like ADFs presented in \cite{bre-ell-str-wal-wol:c:adf-revisited}, which uses an acceptance function for every argument. This could prove to be helpful in finding a ranking with conditional logic. 

\cite{kern2011default} used a similar idea to rank arguments from a \emph{Defeasible Logic Programming} (DeLP), a system, which combines logics programming with defeasible argumentation. They used System Z to identify ``good'' arguments. 

\section*{Acknowledgements}
The research reported here was supported by the Deutsche Forschungsgemeinschaft under grant KE~1413/11-1.

\bibliography{references}
\end{document}